\documentclass{article}

\usepackage[sglblindworkshop, final]{neurips_2025}
\workshoptitle{ML×OR: Mathematical Foundations and Operational Integration of Machine Learning for Uncertainty-Aware Decision-Making}

\usepackage[utf8]{inputenc} 
\usepackage[T1]{fontenc}    
\usepackage{hyperref}       
\usepackage{url}            
\usepackage{booktabs}       
\usepackage{amsfonts}       
\usepackage{nicefrac}       
\usepackage{microtype}      
\usepackage{xcolor}         

\usepackage{algorithm}
\usepackage{algorithmic}

\usepackage{microtype}
\usepackage{graphicx}
\usepackage{subfigure}
\usepackage{hyperref}
\usepackage{amsmath}
\usepackage{amssymb}
\usepackage{mathtools}
\usepackage{amsthm}
\usepackage[font=small, labelfont=bf]{caption}
\usepackage{tcolorbox} 
\usepackage[capitalize,noabbrev]{cleveref}
\newtheorem{theorem}{Theorem}
\newtheorem{assumption}{Assumption}

\title{SOLID: a Framework of Synergizing Optimization and LLMs for Intelligent Decision-Making}

\author{Yinsheng Wang\thanks{Department of Industrial \& Systems Engineering, University of Washington, \texttt{yinshw@uw.edu}}
\and Tario G You\thanks{College of Engineering, University of Washington, \texttt{tariogyou@gmail.com}}
\and Léonard Boussioux\thanks{Michael G. Foster School of Business, University of Washington, \texttt{leobix@uw.edu}}
\and Shan Liu\thanks{Department of Industrial \& Systems Engineering, University of Washington, \texttt{shanliu@uw.edu}}
}

\begin{document}

\maketitle

\begin{abstract}
This paper introduces SOLID (Synergizing Optimization and Large Language Models for Intelligent Decision-Making), a novel framework that integrates mathematical optimization with the contextual capabilities of large language models (LLMs). SOLID facilitates iterative collaboration between optimization and LLMs agents through dual prices and deviation penalties. This interaction improves the quality of the decisions while maintaining modularity and data privacy. The framework retains theoretical convergence guarantees under convexity assumptions, providing insight into the design of LLMs prompt. To evaluate SOLID, we applied it to a stock portfolio investment case with historical prices and financial news as inputs. Empirical results demonstrate convergence under various scenarios and indicate improved annualized returns compared to a baseline optimizer-only method, validating the synergy of the two agents. SOLID offers a promising framework for advancing automated and intelligent decision-making across diverse domains.
\end{abstract}

\section{Introduction}

Mathematical optimization modeling plays a fundamental role in modern decision-making processes, offering rigorous frameworks for complex problem-solving across diverse domains including financial portfolio management, supply chain operations, and healthcare resource allocation. While these models excel at processing quantitative inputs and structured data, they face inherent limitations in directly assimilating unstructured information such as clinical narratives, financial market commentary, and expert analytical reports. Traditional approaches typically convert unstructured input into quantifiable embeddings for downstream predictive models, but often risk compromising decision quality where nuanced contextual understanding is essential \citep{carballo2022tabtext, patil2023survey}.

Large language models (LLMs) have emerged as powerful tools that excel in processing and interpreting vast amounts of contextual and unstructured information, though they inherently lack precise numerical optimization capabilities. LLMs demonstrate considerable potential as decision-making agents, particularly in contexts requiring natural language understanding and complex problem-solving \citep{wasserkrug2024combining, bommasani2021opportunities}. However, they face significant limitations when attempting structured symbolic reasoning and optimization at real-world scale \citep{yang2024large}, with responses often based on statistical pattern matching rather than true mathematical reasoning, leading to inconsistent results when dealing with complex constraints or large solution spaces \citep{imani2023mathprompter}.

Given the complementary strengths of optimization models and LLMs, a natural research question arises: \textit{Can we develop an automated pipeline that effectively integrates decision-making processes from optimization models and LLMs to achieve improved outcomes?} This work proposes SOLID (\textbf{S}ynergizing \textbf{O}ptimization and \textbf{L}LMs for \textbf{I}ntelligent \textbf{D}ecision-Making), which combines the quantitative decision-making strengths of optimization models with the context-awareness capabilities of LLMs. The framework draws inspiration from the alternating direction method of multipliers (ADMM) \citep{boyd2004convex} for inter-agent coordination through dual pricing adjustments and decision deviation penalties.

Our work makes three key contributions: First, we develop a novel framework that combines the numerical precision of optimization models with the contextual understanding and reasoning of LLMs, enabling effective decision-making across structured and unstructured data. Second, we establish a coordinated interaction mechanism, inspired by dual decomposition that facilitates iterative collaboration between optimization and LLM agents through price signals and deviation penalties, ensuring convergence to consensus solutions under appropriate conditions. Third, we demonstrate the practical efficacy of our approach through portfolio optimization experiments that incorporate alternative data sources, showing improved risk-adjusted returns compared to single-agent baselines while maintaining theoretical guarantees where applicable.

\begin{figure*}[htb!]
\centering
\includegraphics[width=0.8\textwidth]{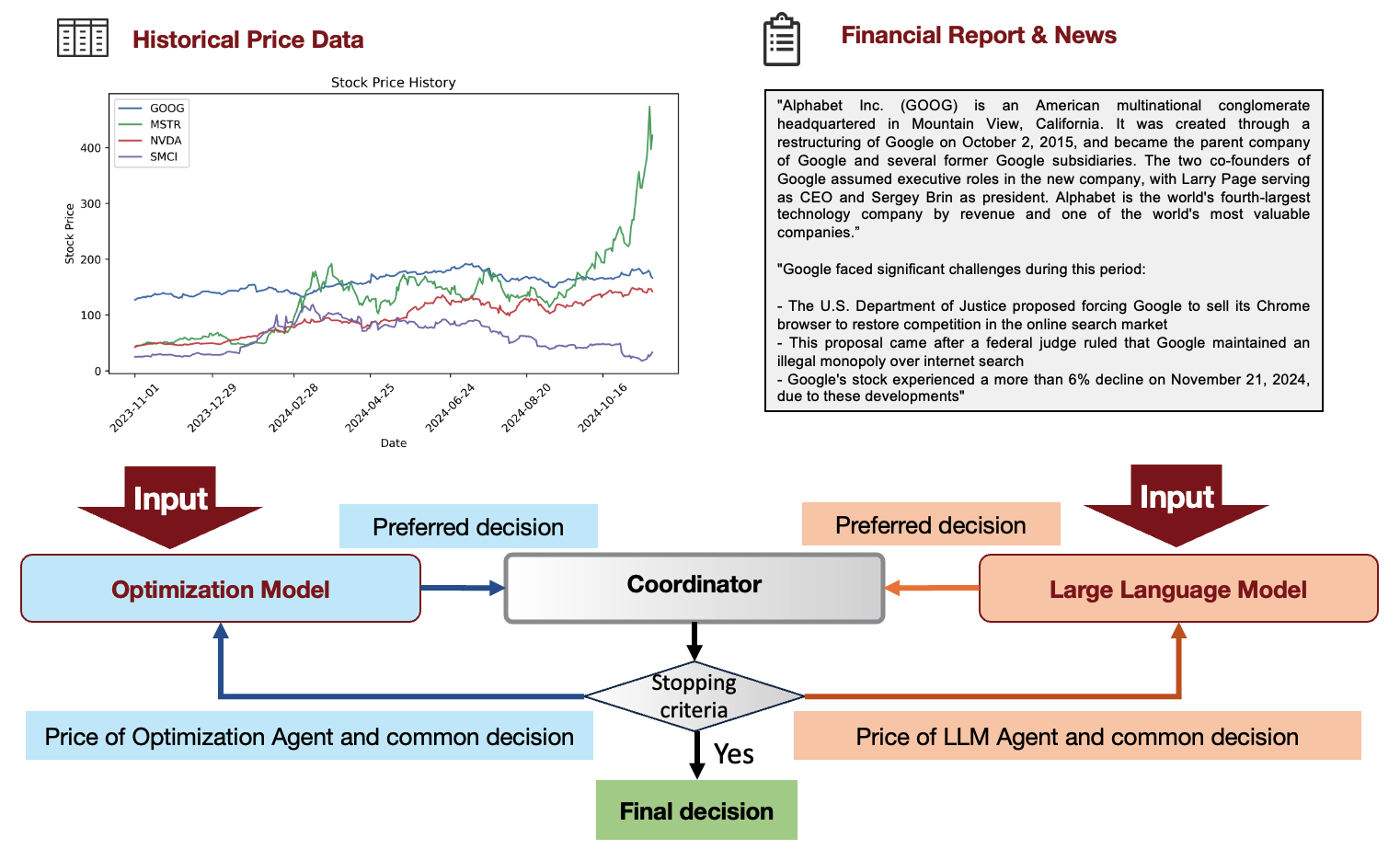}
\caption{Given a decision query from a user, our framework aims to perform optimal decision-making by leveraging the advantages of both the optimization model and the LLM model. We first illustrate the decision-making task by financial investment.}
\label{fig_framework}
\end{figure*}

\section{Framework of SOLID}\label{sec3}

Our framework draws inspiration from ADMM, a distributed convex optimization algorithm that coordinates solutions to smaller sub-problems to solve larger global issues. ADMM decomposes the problem into local variables with consensus constraints, using dual prices and quadratic penalties to achieve coordination \citep{boyd2004convex}. In SOLID, a coordinator compensates or charges optimization and LLM agents for their realized activity values through dual pricing in each iteration. As illustrated in Figure \ref{fig_framework}, at each iteration, the coordinator transmits the current public plan $x^{(k)}$ and dual prices $\lambda_{llm}^{(k)}$, $\lambda_{opt}^{(k)}$ to each agent. This information creates augmented utility functions that pay agents for activities while penalizing deviations from consensus. Each agent proposes preferred decisions maximizing their augmented utility, and the coordinator performs reconciliation to produce a new common decision. This process repeats until convergence. The framework is formalized in Algorithm \ref{alg_solid}.

\begin{algorithm}[htb!]
\caption{Framework of SOLID}
\label{alg_solid}
\begin{algorithmic}
\STATE {\bfseries Input:} Consistency set $\chi$, max number of iterations $K$, step size $\rho$, $k = 1$.
\STATE {\bfseries Initialization:} Initial public decision variables $x^{(0)}$ and dual variables $\lambda_{llm}^{(0)}$ and $\lambda_{opt}^{(0)}$.
\REPEAT
    \STATE \textbf{Optimization agent} updates its preferred plans $\tilde{x}_{opt}^{(k)}$ by maximizing its augmented utility function:
    \begin{align}
    \tilde{x}_{opt}^{(k)} = \arg \max_{x \in \mathcal{X}}\left\{u_{opt}(x)+x^{\top} \lambda_{opt}^{(k-1)}-\frac{\rho}{2}\left\|x-x^{(k-1)}\right\|^2\right\}
    \end{align}
    \STATE \textbf{LLM agent} functions as a decision-making agent and updates its preferred plans $\tilde{x}_{llm}^{(k)}$:
    \begin{align}
    \tilde{x}_{llm}^{(k)} = \arg \max_{x \in \mathcal{X}}\left\{u_{llm}(x, \lambda_{llm}^{(k-1)}, x^{(k-1)})\right\}
    \end{align}
    \STATE \textbf{Coordinator} updates activity prices:
    \begin{align}
    \lambda_{agent}^{(k)} &= \lambda_{agent}^{(k-1)} - \rho \left( \tilde{x}_{agent}^{(k)} - x^{(k-1)} \right) \quad agent \in \{llm, opt\}
    \end{align}
    \STATE \textbf{Coordinator} updates public decision variables as follows:
    \begin{align}
    x^{(k)} = \arg \min_{x \in \mathcal{X}} \left\{ x^{\top} \left(\lambda_{opt}^{(k)} + \lambda_{llm}^{(k)}\right) + \frac{\rho}{2} \left( \left\|\tilde{x}_{opt}^{(k)} - x\right\|^2 + \left\|\tilde{x}_{llm}^{(k)} - x\right\|^2 \right) \right\}
    \end{align}
    \STATE $k = k + 1$
\UNTIL{$k = K$ {\bfseries or} primal/dual residuals are small enough}
\STATE {\bfseries Output:} $x^{(k)}$, $\lambda_{llm}^{(k)}$, $\lambda_{opt}^{(k)}$.
\end{algorithmic}
\end{algorithm}

SOLID provides two key advantages. First, it enables automated decision-making by systematically coordinating optimization models with LLMs through price signals, thereby eliminating the need for manual intervention. Second, the modular design allows independent development of each component on separate datasets, preserving data privacy while enabling collaborative decision-making.

\paragraph{Theoretical Property and Considerations} SOLID integrates mathematical optimization with LLMs through an iterative mechanism where each component operates as an independent decision-maker. Under standard convexity assumptions (proper, closed, convex functions; feasible problem; nonempty minimizer set), Algorithm \ref{alg_solid} converges to optimal solutions. However, real-world systems often violate convexity assumptions due to discrete variables or non-convex objectives.

The introduction of LLM agents raises similar challenges. LLMs struggle with numerical calculations but excel at semantic descriptions and ranking tasks \citep{avnat2024performance}. To address resulting non-convexity, we adopt higher-level abstractions providing accurate surrogates for continuous decisions. SOLID equips LLM agents with economic concepts to enhance reasoning: dual prices $\lambda_{llm}$ represent economic costs of preferred decisions, while current public variables impose deviation penalties, ensuring alignment and feasibility.

\section{Case Study: Portfolio Optimization with Alternative Data}\label{sec5}

We demonstrate SOLID's effectiveness through portfolio optimization integrating historical prices with qualitative market intelligence from financial news, combining quantitative and unstructured data sources for improved investment strategies.

\subsection{Experimental Setup}

\paragraph{Optimization Agent} We adopt the classic Markowitz Mean-Variance Portfolio Optimization model as the optimization agent, minimizing portfolio risk for a given target return. The optimization problem determines optimal portfolio weights $\mathbf{w} = [w_1, w_2, \dots, w_n]^{\prime}$: \(\min _{\mathbf{w}} \frac{1}{2} \mathbf{w}^{\prime} \boldsymbol{\Sigma} \mathbf{w} \quad \text{s.t.} \quad \mathbf{w}^{\prime} \mu = p, \quad \mathbf{w}^{\prime} \mathbf{1} = 1,\) where $\boldsymbol{\Sigma}$ is the covariance matrix, $\mu$ is the expected returns vector, and $p$ is the target return.

\paragraph{LLM Agent} The LLM considers recent news and stock prices, constructed through prompt engineering to make informed decisions. We discretize LLM decisions into semantic levels: "Very High", "High", "Somewhat High", "Neutral", "Somewhat Low", "Low", and "Very Low", mapping to numerical values from 0.6 to 0 with 0.1 increments.

We selected 60 NASDAQ stocks spanning 10 industries and conducted month-by-month portfolio adjustments throughout 2024. We compare five strategies: pure optimization (\textbf{OPT}), pure LLM (\textbf{LLM}), SOLID framework (\textbf{LLM+OPT}), simple averaging (\textbf{AVG}), and sparse versions enforcing sparsity in LLM allocations. Our primary metrics are portfolio value tracking and risk (total variance) monitoring. We primarily use ChatGPT-4o-mini with comparisons across GPT-4o and o1-mini.

\subsection{Results}

\begin{figure*}[htb!]
\centering
\includegraphics[width=0.9\textwidth]{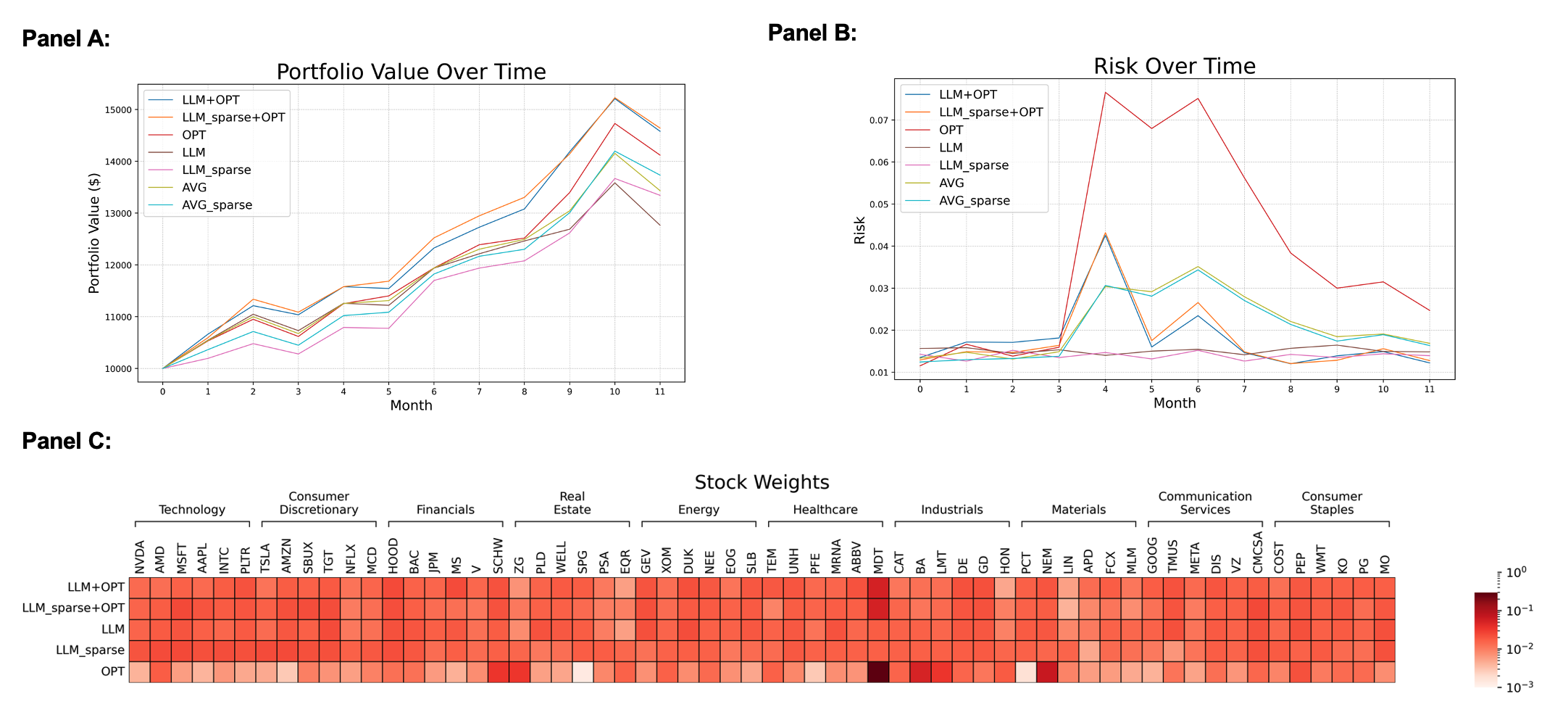}
\caption{Panel A: the total portfolio value change by month; Panel B: Risk evaluation by month; Panel C: average stock weights under 5 strategies by month.}
\label{fig_portfolio}
\end{figure*}

\paragraph{Portfolio Performance} Figure \ref{fig_portfolio} Panel A shows portfolio value evolution from \$10,000 over 12 months. SOLID variants (LLM+OPT and LLM\_sparse+OPT) consistently achieve the highest portfolio values, followed by OPT, while pure LLM strategies yield more moderate returns. Panel B demonstrates that OPT exhibits pronounced volatility spikes around months 4-5, while LLM\_sparse maintains consistently lower volatility. SOLID strategies (LLM+OPT, LLM\_sparse+OPT) achieve intermediate risk levels, effectively balancing return potential with risk control. Panel C's heatmap reveals that OPT frequently exhibits extreme allocations to specific stocks (e.g., MDT), corresponding to elevated risk levels. LLM-based approaches display more diversified allocations with stock weights distributed more evenly.

\paragraph{Agent Convergence} Figure \ref{fig_converge} demonstrates SOLID's coordination process with the AMZN stock example. The framework achieves exact convergence at certain time points (months 1 and 7), while in other instances, decisions progressively align, indicating a tendency toward agreement despite LLM complexity.

\paragraph{Model Comparisons} Across ChatGPT-4o-mini, 4o, and o1-mini, SOLID pipelines consistently achieve optimal risk-return balance. With 4o-mini, both LLM+OPT and LLM\_sparse+OPT outperform other methods. With 4o, LLM\_sparse+OPT achieves exceptionally high returns. The results validate SOLID's robustness across different LLM architectures.

These results demonstrate that SOLID enhances portfolio performance relative to purely LLM-based strategies while mitigating excessive risk associated with pure optimization, validating the synergy between optimization precision and LLM contextual understanding.

\section{Conclusions}\label{sec6}
We propose SOLID, a decision-making process that integrates optimization models and LLMs within a distributed optimization framework. Empirical results provide insights into designing effective prompting strategies for LLMs. A case study in finance highlights the advantages of our approach compared to baseline methods, including using only optimization models or LLMs. The modular and interpretable structure of SOLID makes it adaptable to a wide range of decision-making systems in critical domains.

\bibliographystyle{plainnat}  
\bibliography{references}     

@article{bommasani2021opportunities,
  title={On the opportunities and risks of foundation models},
  author={Bommasani, Rishi and Hudson, Drew A and Adeli, Ehsan and Altman, Russ and Arora, Simran and von Arx, Sydney and Bernstein, Michael S and Bohg, Jeannette and Bosselut, Antoine and Brunskill, Emma and others},
  journal={arXiv preprint arXiv:2108.07258},
  year={2021}
}

@article{ahmaditeshnizi2024optimus,
  title={OptiMUS: Scalable Optimization Modeling with (MI) LP Solvers and Large Language Models},
  author={AhmadiTeshnizi, Ali and Gao, Wenzhi and Udell, Madeleine},
  journal={arXiv preprint arXiv:2402.10172},
  year={2024}
}

@article{wasserkrug2024large,
  title={From Large Language Models and Optimization to Decision Optimization CoPilot: A Research Manifesto},
  author={Wasserkrug, Segev and Boussioux, Leonard and Hertog, Dick den and Mirzazadeh, Farzaneh and Birbil, Ilker and Kurtz, Jannis and Maragno, Donato},
  journal={arXiv preprint arXiv:2402.16269},
  year={2024}
}

@incollection{wasserkrug2024combining,
  title={Combining Large Language Models and OR/MS to Make Smarter Decisions},
  author={Wasserkrug, Segev and Boussioux, L{\'e}onard and Sun, Wei},
  booktitle={Tutorials in Operations Research: Smarter Decisions for a Better World},
  pages={1--49},
  year={2024},
  publisher={INFORMS}
}

@article{imani2023mathprompter,
  title={Mathprompter: Mathematical reasoning using large language models},
  author={Imani, Shima and Du, Liang and Shrivastava, Harsh},
  journal={arXiv preprint arXiv:2303.05398},
  year={2023}
}

@book{boyd2004convex,
  author = {Stephen Boyd and Lieven Vandenberghe},
  title = {Convex Optimization},
  year = {2004},
  publisher = {Cambridge University Press},
  address = {Cambridge},
  note = {ISBN 978-0-521-83558-1}
}

@article{wei2022chain,
  title={Chain-of-thought prompting elicits reasoning in large language models},
  author={Wei, Jason and Wang, Xuezhi and Schuurmans, Dale and Bosma, Maarten and Xia, Fei and Chi, Ed and Le, Quoc V and Zhou, Denny and others},
  journal={Advances in neural information processing systems},
  volume={35},
  pages={24824--24837},
  year={2022}
}

@article{yao2024tree,
  title={Tree of thoughts: Deliberate problem solving with large language models},
  author={Yao, Shunyu and Yu, Dian and Zhao, Jeffrey and Shafran, Izhak and Griffiths, Tom and Cao, Yuan and Narasimhan, Karthik},
  journal={Advances in Neural Information Processing Systems},
  volume={36},
  year={2024}
}

@article{long2023large,
  title={Large Language Model Guided Tree-of-Thought},
  author={Long, Jieyi},
  journal={arXiv preprint arXiv:2305.08291},
  year={2023}
}

@article{ye2023large,
  title={Large language model as autonomous decision maker},
  author={Ye, Yining and Cong, Xin and Qin, Yujia and Lin, Yankai and Liu, Zhiyuan and Sun, Maosong},
  journal={arXiv preprint arXiv:2308.12519},
  year={2023}
}

@article{hao2023reasoning,
  title={Reasoning with language model is planning with world model},
  author={Hao, Shibo and Gu, Yi and Ma, Haodi and Hong, Joshua Jiahua and Wang, Zhen and Wang, Daisy Zhe and Hu, Zhiting},
  journal={arXiv preprint arXiv:2305.14992},
  year={2023}
}

@article{liu2024dellma,
  title={DeLLMa: A Framework for Decision Making Under Uncertainty with Large Language Models},
  author={Liu, Ollie and Fu, Deqing and Yogatama, Dani and Neiswanger, Willie},
  journal={arXiv preprint arXiv:2402.02392},
  year={2024}
}

@misc{perplexity,
  author = {Perplexity AI, Inc.},
  title = {Perplexity AI},
  year = {2024},
  url = {https://www.perplexity.ai/},
  note = {Accessed: 2024-11-24}
}

@misc{yahoo_finance_api,
  author       = {{Yahoo Finance}},
  title        = {Yahoo Finance API},
  year         = {2025},
  url          = {https://finance.yahoo.com},
  note         = {Accessed: 2025-01-07}
}

@article{boyd2011distributed,
  title={Distributed optimization and statistical learning via the alternating direction method of multipliers},
  author={Boyd, Stephen and Parikh, Neal and Chu, Eric and Peleato, Borja and Eckstein, Jonathan and others},
  journal={Foundations and Trends{\textregistered} in Machine learning},
  volume={3},
  number={1},
  pages={1--122},
  year={2011},
  publisher={Now Publishers, Inc.}
}

@article{guo2023towards,
  title={Towards optimizing with large language models},
  author={Guo, Pei-Fu and Chen, Ying-Hsuan and Tsai, Yun-Da and Lin, Shou-De},
  journal={arXiv preprint arXiv:2310.05204},
  year={2023}
}

@inproceedings{yang2024large,
title={Large Language Models as Optimizers},
author={Chengrun Yang and Xuezhi Wang and Yifeng Lu and Hanxiao Liu and Quoc V Le and Denny Zhou and Xinyun Chen},
booktitle={The Twelfth International Conference on Learning Representations},
year={2024},
url={https://openreview.net/forum?id=Bb4VGOWELI}
}

@article{avnat2024performance,
  title={Performance of large language models in numerical vs. semantic medical knowledge: Benchmarking on evidence-based Q\&As},
  author={Avnat, Eden and Levy, Michal and Herstain, Daniel and Yanko, Elia and Joya, Daniel Ben and Katz, Michal Tzuchman and Eshel, Dafna and Laros, Sahar and Dagan, Yael and Barami, Shahar and others},
  journal={arXiv preprint arXiv:2406.03855},
  year={2024}
}

@inproceedings{alavian2017improving,
  title={Improving ADMM-based optimization of mixed integer objectives},
  author={Alavian, Alborz and Rotkowitz, Michael C},
  booktitle={2017 51st Annual Conference on Information Sciences and Systems (CISS)},
  pages={1--6},
  year={2017},
  organization={IEEE}
}

@article{diamond2018general,
  title={A general system for heuristic minimization of convex functions over non-convex sets},
  author={Diamond, Steven and Takapoui, Reza and Boyd, Stephen},
  journal={Optimization Methods and Software},
  volume={33},
  number={1},
  pages={165--193},
  year={2018},
  publisher={Taylor \& Francis}
}

@article{wang2019global,
  title={Global convergence of ADMM in nonconvex nonsmooth optimization},
  author={Wang, Yu and Yin, Wotao and Zeng, Jinshan},
  journal={Journal of Scientific Computing},
  volume={78},
  pages={29--63},
  year={2019},
  publisher={Springer}
}

@article{carballo2022tabtext,
  title={TabText: A Flexible and Contextual Approach to Tabular Data Representation},
  author={Carballo, Kimberly Villalobos and Na, Liangyuan and Ma, Yu and Boussioux, L{\'e}onard and Zeng, Cynthia and Soenksen, Luis R and Bertsimas, Dimitris},
  journal={arXiv preprint arXiv:2206.10381},
  year={2022}
}

@article{patil2023survey,
  title={A survey of text representation and embedding techniques in nlp},
  author={Patil, Rajvardhan and Boit, Sorio and Gudivada, Venkat and Nandigam, Jagadeesh},
  journal={IEEE Access},
  volume={11},
  pages={36120--36146},
  year={2023},
  publisher={IEEE}
}

@article{silva2021business,
  title={Business analytics in Industry 4.0: A systematic review},
  author={Silva, Ant{\'o}nio Jo{\~a}o and Cortez, Paulo and Pereira, Carlos and Pilastri, Andr{\'e}},
  journal={Expert systems},
  volume={38},
  number={7},
  pages={e12741},
  year={2021},
  publisher={Wiley Online Library}
}

@article{dantzig2002linear,
  title={Linear programming},
  author={Dantzig, George B},
  journal={Operations research},
  volume={50},
  number={1},
  pages={42--47},
  year={2002},
  publisher={INFORMS}
}

@article{wolsey2007mixed,
  title={Mixed integer programming},
  author={Wolsey, Laurence A},
  journal={Wiley Encyclopedia of Computer Science and Engineering},
  pages={1--10},
  year={2007},
  publisher={Wiley Online Library}
}

@book{luenberger1984linear,
  title={Linear and nonlinear programming},
  author={Luenberger, David G and Ye, Yinyu and others},
  volume={2},
  year={1984},
  publisher={Springer}
}

@article{tang2023learning,
  title={Learning optimal and fair policies for online allocation of scarce societal resources from data collected in deployment},
  author={Tang, Bill and Ko{\c{c}}yi{\u{g}}it, {\c{C}}a{\u{g}}{\i}l and Rice, Eric and Vayanos, Phebe},
  journal={arXiv preprint arXiv:2311.13765},
  year={2023}
}

@article{missbauer2011optimization,
  title={Optimization models of production planning problems},
  author={Missbauer, Hubert and Uzsoy, Reha},
  journal={Planning Production and Inventories in the Extended Enterprise: A State of the Art Handbook, Volume 1},
  pages={437--507},
  year={2011},
  publisher={Springer}
}

@article{zhang2023automl,
  title={Automl-gpt: Automatic machine learning with gpt},
  author={Zhang, Shujian and Gong, Chengyue and Wu, Lemeng and Liu, Xingchao and Zhou, Mingyuan},
  journal={arXiv preprint arXiv:2305.02499},
  year={2023}
}

@article{ye2024reevo,
  title={Reevo: Large language models as hyper-heuristics with reflective evolution},
  author={Ye, Haoran and Wang, Jiarui and Cao, Zhiguang and Berto, Federico and Hua, Chuanbo and Kim, Haeyeon and Park, Jinkyoo and Song, Guojie},
  journal={arXiv preprint arXiv:2402.01145},
  year={2024}
}

@article{fernando2023promptbreeder,
  title={Promptbreeder: Self-referential self-improvement via prompt evolution},
  author={Fernando, Chrisantha and Banarse, Dylan and Michalewski, Henryk and Osindero, Simon and Rockt{\"a}schel, Tim},
  journal={arXiv preprint arXiv:2309.16797},
  year={2023}
}

@article{brahmachary2024large,
  title={Large Language Model-Based Evolutionary Optimizer: Reasoning with elitism},
  author={Brahmachary, Shuvayan and Joshi, Subodh M and Panda, Aniruddha and Koneripalli, Kaushik and Sagotra, Arun Kumar and Patel, Harshil and Sharma, Ankush and Jagtap, Ameya D and Kalyanaraman, Kaushic},
  journal={arXiv preprint arXiv:2403.02054},
  year={2024}
}

@misc{deepseekai2025deepseekr1incentivizingreasoningcapability,
      title={DeepSeek-R1: Incentivizing Reasoning Capability in LLMs via Reinforcement Learning}, 
      author={DeepSeek-AI},
      year={2025},
      eprint={2501.12948},
      archivePrefix={arXiv},
      primaryClass={cs.CL},
      url={https://arxiv.org/abs/2501.12948}, 
}

@misc{openai_learning_to_reason_2024,
  author       = {{OpenAI}},
  title        = {Learning to Reason with LLMs},
  year         = {2024},
  url          = {https://openai.com/index/learning-to-reason-with-llms/}
}

\appendix
\onecolumn

\section{Related Work}\label{sec2}

Our work sits at the intersection of optimization modeling, the emerging research of the LLM's reasoning capabilities, and the combinations of optimization and LLMs for enhanced decision-making. 

\paragraph{Optimization Modeling:} Optimization is ubiquitous in modern business analytics, enabling organizations to find optimal solutions to complex problems by systematically evaluating scenarios and constraints \citep{silva2021business}. Unlike predictive analytics, optimization recommends specific actions to achieve desired outcomes. Techniques such as linear programming \cite{dantzig2002linear}, integer programming \cite{wolsey2007mixed}, and non-linear programming \cite{luenberger1984linear} are widely applied to resource allocation \cite{tang2023learning}, production planning \cite{missbauer2011optimization}, and distribution management, offering actionable insights that enhance efficiency, reduce costs, and improve decision-making across industries.

\paragraph{LLM as a Decision-Making Agent:} Decision-making with LLMs has recently gained traction \citep{ahmaditeshnizi2024optimus, wasserkrug2024large, wasserkrug2024combining}. Chain-of-Thought (CoT) prompting \citep{wei2022chain} initiated a wave of research that deconstructs complex multistep reasoning challenges, including decision-making, into modular sub-problems. For example, Tree-of-Thought (ToT) prompting \citep{yao2024tree, long2023large} expands on CoT by incorporating a tree-search approach, optimizing reasoning paths based on external feedback. Subsequent research has refined ToT through improved search algorithms, self-induced feedback, and tool integration \citep{ye2023large, hao2023reasoning}. Recent work known as DeLLMa introduces a framework for decision-making under uncertainty \citep{liu2024dellma}, revealing the inherent challenges even when carefully structured chains attempt to replicate classical decision-making processes. However, this body of work has not yet explored the combined advantages of prescriptive models and LLMs for decision-making.

\paragraph{Combining Optimization and LLMs:}

Recent studies on the integration of LLMs and optimization algorithms include at least two primary approaches. The first is leveraging LLMs to formulate or solve optimization problems. For example, OptiMUS has been developed to translate natural language descriptions into mixed-integer linear programming (MILP) formulations \cite{ahmaditeshnizi2024optimus}, and LLMs have also been explored as black-box solvers for optimization tasks \cite{brahmachary2024large}. The second approach involves applying optimization techniques to improve the performance of LLMs. This includes methods such as structural pruning and neural architecture search, as demonstrated in AutoML-GPT \cite{zhang2023automl}. ReEvo \cite{ye2024reevo}, PromptBreeder \cite{fernando2023promptbreeder} and Optimization by Prompting (OPRO) \cite{yang2024large}, aim to automate the process of prompt optimization. Recent DeepSeek-R1 \cite{deepseekai2025deepseekr1incentivizingreasoningcapability} uses large-scale reinforcement learning to improve reasoning performance, and achieves performance comparable to OpenAI-o1 \cite{openai_learning_to_reason_2024} on reasoning tasks. These advancements highlight the growing importance of combining LLMs and optimization techniques to address both computational and practical challenges.

\section{Details of ADMM Algorithm}

Our framework draws inspiration from ADMM. ADMM is a simple but powerful algorithm that is particularly well suited to distributed convex optimization. It functions as a decomposition-coordination method, coordinating solutions to smaller, localized sub-problems to find a solution for a larger, global issue. Here we describe a general formulation of the problem \citep{boyd2004convex}.

We first consider the case with a single global variable, with the objective split into \(N\) parts: \(f(x)=\sum_{i=1}^N f_i(x),\)
where \(x \in \mathbf{R}^n\), and \(f_i: \mathbf{R}^n \rightarrow \mathbf{R} \cup\{+\infty\}\) are convex. We refer to \(f_i\) as the \(i\) th term in the objective. Each term can also encode constraints by assigning \(f_i(x)=+\infty\) when a constraint is violated. The goal is to solve the above problem in such a way that each term can be handled separately.

This problem can be rewritten with local variables \(x_i \in \mathbf{R}^n\) and a common global variable \(z\) :
\[
\begin{array}{ll}
\operatorname{minimize} & \sum_{i=1}^N f_i\left(x_i\right) \\
\text { subject to } & x_i-z=0, \quad i=1, \ldots, N
\end{array}
\]

The resulting ADMM algorithm is the following:
\[
\begin{aligned}
x_i^{k+1} & :=\underset{x_i}{\operatorname{argmin}}\left(f_i\left(x_i\right)+y_i^{k T}\left(x_i-z^k\right)+\frac{\rho}{2}\left\|x_i-z^k\right\|_2^2\right) \\
z^{k+1} & :=\frac{1}{N} \sum_{i=1}^N\left(x_i^{k+1}+(1 / \rho) y_i^k\right) \\
y_i^{k+1} & :=y_i^k+\rho\left(x_i^{k+1}-z^{k+1}\right) .
\end{aligned}
\]

\section{Theoretical Property Elaborations}

The foundational concept of SOLID involves the integration of a mathematical optimization model with LLMs through an iterative mechanism wherein each component operates as an independent decision-maker. Within each iteration, the problem presented to the decision-maker is modified to incorporate a bias that assigns weight to the "joint" decision derived from the decisions of both entities. This methodology draws inspiration from the ADMM. In this section, our focus is on evaluating whether any formal characteristics of ADMM are helpful to guide us on improving the performance of SOLID framework. Moreover, given that LLMs are often heuristic and ad hoc in nature, it is important to maintain transparency regarding the roles played by ADMM and each agent.

Let \(u_{opt}(x)\) be the utility function of the optimization agent and \(u_{opt}(z)\) be the utility function of the LLM agent. Consider the convex optimization problem
\begin{equation}
\begin{aligned}
\min_{x,z}\quad & f(x) + g(z) \quad\text{s.t.}\quad & x = z,
\end{aligned}
\label{opt_formulation}
\end{equation}

where we identify 
\(f(x) := u_{\mathrm{opt}}(x)\), and \(g(z) := u_{\mathrm{llm}}(z)\).

It can be shown that the Algorithm~\ref{alg_solid} converges to a solution of this problem under the following standard assumptions.

\begin{assumption}\label{assumption:1}
\begin{enumerate}
\item[(A1)] The functions $f$ and $g$ are proper, closed, and convex.
\item[(A2)] The problem is \emph{feasible}, i.e., there exists $(x,z)$ with $x=z$ that achieves a finite value of $f(x) + g(z)$.
\item[(A3)] The set of minimizers is nonempty and at least one optimal solution $(x^\star,z^\star)$ exists.
\end{enumerate}
\end{assumption}

We rewrite the constraint $x = z$ as $x - z = 0$, and let $\lambda$ be the dual variable (Lagrange multiplier). The augmented Lagrangian of this problem is
\[
\mathcal{L}_{\rho}(x,z,\lambda)=
f(x) + g(z) 
+ \langle \lambda, x - z \rangle
+ \frac{\rho}{2}\|x - z\|^2,
\]
where $\rho > 0$ is a penalty parameter.

The ADMM iteration updates $(x^k,z^k,\lambda^k)$ as follows:
\begin{align*}
x^{k+1} 
&= \arg\min_{x} 
\bigl\{ 
   f(x) + \langle \lambda^k,\, x - z^k \rangle 
+ \tfrac{\rho}{2}\|x - z^k\|^2 
\bigr\},
\\
z^{k+1} 
&= \arg\min_{z} 
\bigl\{
   g(z) 
- \langle \lambda^k,\, z - x^{k+1} \rangle 
+ \tfrac{\rho}{2}\|x^{k+1} - z\|^2
\bigr\},
\\
\lambda^{k+1}
&= 
\lambda^k +
\rho \bigl( x^{k+1} - z^{k+1} \bigr).
\end{align*}

\begin{theorem}[Convergence of ADMM]\label{thm:ADMM}
Suppose Assumption~\ref{assumption:1} holds. Then any sequence $\{(x^k,z^k,\lambda^k)\}$ produced by the ADMM iteration above is bounded (or has a convergent subsequence). Moreover, every limit point $(x^\star,z^\star,\lambda^\star)$ of this sequence satisfies: \textbf{1. Primal feasibility:} $x^\star = z^\star$; \textbf{2. Dual feasibility (KKT condition):}
\(-\lambda^\star \;\in\;\partial f(x^\star), \quad \lambda^\star \;\in\;\partial g(z^\star).\)

Hence, $(x^\star,z^\star)$ is an optimal solution of the original problem, and $\lambda^\star$ is an optimal dual variable.
\end{theorem}

\begin{proof}[Sketch of Proof]
We outline the standard steps; detailed arguments can be found in reference on ADMM \cite{boyd2011distributed}. Consider the problem \ref{opt_formulation} with $f$ and $g$ proper, closed, convex, and at least one optimal solution $(x^\star,z^\star)$ satisfying $x^\star=z^\star$.  The ADMM algorithm forms the augmented Lagrangian 
\(
\mathcal{L}_{\rho}(x,z,\lambda) 
= f(x) + g(z) + \langle \lambda,\,x-z\rangle + \tfrac{\rho}{2}\|x-z\|^2
\)
and iteratively updates $x,z,\lambda$ by separately minimizing over $x$ and $z$ (with quadratic regularization) and then performing a gradient ascent step on $\lambda$. Since each subproblem is strictly convex and $f$ and $g$ are closed, each update is well-defined.  A standard boundedness argument on the augmented Lagrangian ensures the sequences $(x^k,z^k,\lambda^k)$ are contained in a region from which a convergent subsequence can be extracted. Passing to the limit in the stationarity conditions for $x$ and $z$ and using the dual update $\lambda^{k+1} = \lambda^k + \rho(x^{k+1}-z^{k+1})$ shows that the primal residual $x^{k+1}-z^{k+1}$ and the dual residual converge to zero, yielding $x^\star=z^\star$ and subgradient conditions $-\lambda^\star\in\partial f(x^\star)$, $\lambda^\star\in\partial g(z^\star)$. These conditions are precisely the KKT conditions for the original problem, so $(x^\star,z^\star)$ is optimal, and thus ADMM converges to a solution.
\end{proof}

The proof relies strictly on Assumption~\ref{assumption:1}. However, real-world systems often violate these assumptions. While conditions (A2) and (A3) in Assumption~\ref{assumption:1} are typically straightforward to satisfy—assuming the existence of a feasible solution that achieves the optimal objective—many problems remain non-convex. This non-convexity may stem from either a non-convex feasible region (e.g., involving discrete decision variables) or a non-convex objective function. In such cases, numerical experiments reveal that if one or more agents exhibit non-convex behavior, convergence cannot always be guaranteed. Many of the related studies \cite{alavian2017improving, wang2019global, diamond2018general} address this by relaxing the assumptions and establishing new convergence guarantees. A common approach is to solve a convexified version of the problem during interaction with the coordinator, subsequently adjusting the solution to meet feasibility conditions.

In the SOLID framework, the introduction of LLM agents raises similar challenges. First, LLMs, which rely on probabilistic patterns, struggle with numerical calculations. However, they excel at interpreting semantic descriptions such as ``high", ``medium," and ``low" \cite{avnat2024performance}, and they perform well in ranking tasks. In the SOLID framework, such semantic descriptions correspond to discrete decision variables. To mitigate the resulting non-convexity, we adopt a higher-level abstraction that provides a more accurate surrogate for continuous decisions. Second, LLMs have demonstrated optimization potential. For example, studies \cite{guo2023towards, yang2024large} have shown that LLMs can deliver good-quality solutions to classic optimization problems, such as linear regression and the traveling salesman problem, especially in small scales.

The SOLID framework equips LLM agents with concise economic concepts to enhance reasoning. Similarly to the augmented utility function used by optimization agents - which includes a dual price $\lambda_{opt}$ and the current public decision variables - we provide LLM agents with their dual price $\lambda_{llm}$. This dual price represents the economic cost associated with the preferred LLM agent decisions. Additionally, the framework incorporates the current public decision variables to impose penalties for deviations from public decisions, ensuring alignment and feasibility.

\section{Detailed Experimental Setup}\label{app:setup}

\subsection{Data Collection}

For collecting historical price data, we adopt Yahoo Finance API \cite{yahoo_finance_api}. We choose to use Perplexity.ai \cite{perplexity} to gather news for LLM input due to its ability to synthesize large volumes of up-to-date information, including news headlines and industry trends.

\subsection{Stock Selection}

For the experiments, we selected 60 NASDAQ stocks spanning 10 different industries, providing comprehensive representation of diverse sectors and ensuring generalizability of portfolio analysis:

\begin{table}[h]
    \centering
    \renewcommand{\arraystretch}{1.2}
    \begin{tabular}{ll}
        \toprule
        \textbf{Industry} & \textbf{Stocks} \\
        \midrule
        Technology & NVDA, AMD, MSFT, AAPL, INTC, PLTR \\
        Consumer Discretionary & TSLA, AMZN, SBUX, TGT, NFLX, MCD \\
        Financials & HOOD, BAC, JPM, MS, V, SCHW \\
        Real Estate & ZG, PLD, WELL, SPG, PSA, EQR \\
        Energy & GEV, XOM, DUK, NEE, EOG, SLB \\
        Healthcare & TEM, UNH, PFE, MRNA, ABBV, MDT \\
        Industrials & CAT, BA, LMT, DE, GD, HON \\
        Materials & PCT, NEM, LIN, APD, FCX, MLM \\
        Communication Services & GOOG, TMUS, META, DIS, VZ, CMCSA \\
        Consumer Staples & COST, PEP, WMT, KO, PG, MO \\
        \bottomrule
    \end{tabular}
    \caption{Stock Tickers by Industry}
    \label{tab:stock_industries}
\end{table}

\subsection{Detailed Backtesting Procedure}

\textbf{Timeframe:} We adopt month-by-month portfolio adjustments and perform model resolution under this timeframe to better capture valuable news and stock price trends. The experiments were carried out over the whole year of 2024, with 12 time periods for each separate month. In our experiments, we use the stock price from the last day as input for the optimization agents and obtain significant news from the entire period for each period.

\textbf{Evaluation Metrics:} Our primary metric is the dynamic tracking of the total portfolio value, with monthly adjustments to the investment weights. To evaluate risk control performance, we monitor the total variance of the portfolio, which aligns with the objective of the optimization model throughout the experiments.

\textbf{Baseline Benchmark Details:} We establish a comprehensive baseline benchmark consisting of portfolio weights derived from a pure optimizer (\textbf{OPT}) and a pure LLM-based approach (\textbf{LLM}), comparing them against a portfolio optimized using the SOLID framework (\textbf{LLM+OPT}). Furthermore, we compute the simple average of the weights from \textbf{LLM} and \textbf{OPT} as another benchmark (\textbf{AVG}). Based on these methods, we enhance the interpretability of LLM-driven decisions by enforcing sparsity, achieved by directing the LLM agent to \textit{aim for sparsity in the final allocation} and setting zero weights for stocks in which it has low confidence. The corresponding sparse versions are denoted as \textbf{LLM\_sparse}, \textbf{LLM\_sparse+OPT}, and \textbf{AVG\_sparse}.

\subsection{LLM Agent Design Details}

We implement a Chain-of-Thought (CoT) style of prompt engineering \citep{wei2022chain}. Although there has not been a standardized CoT pipeline for decision-making tasks, we design a prompting chain consisting of four steps: (1) read news articles and reflect on how they might affect each stock; (2) pass the public information of the current iteration and decision price, prompting the LLM to collaborate with an optimization agent; (3) finalize the recommendation with a confidence level; and (4) make a final decision based on the previous steps.

To ensure consistency, the LLM temperature is set to 0, encouraging the LLM to exploit the solution space around the previously found solutions and make small adaptations.

\subsection{Convergence and Model Comparisons Results}

\begin{figure}[htb!]
\centering
\includegraphics[width=0.7\textwidth]{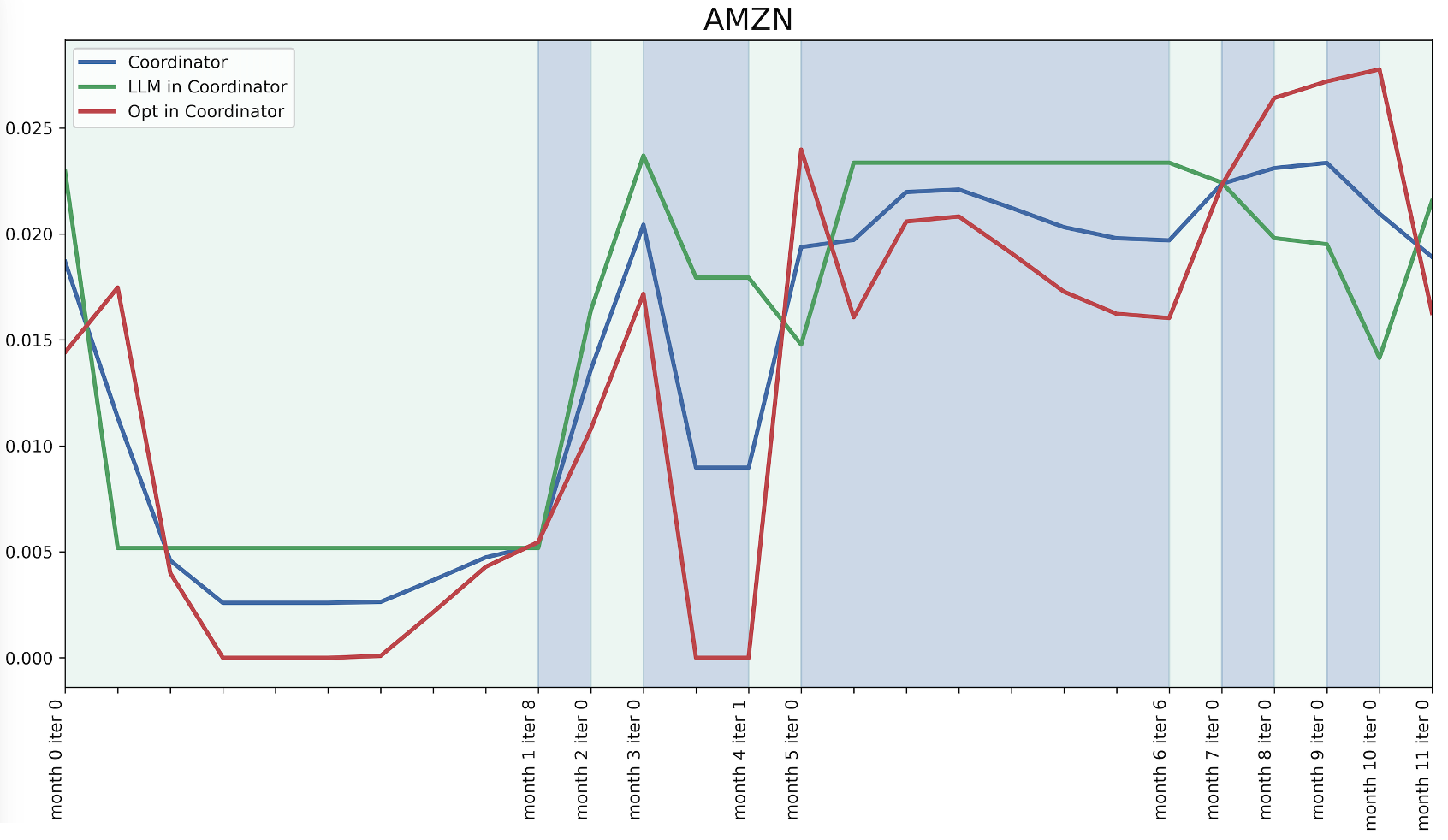}
\caption{The coordination process of optimization and LLM agents in SOLID for exemplary stock - AMZN.}
\label{fig_converge}
\end{figure}

\begin{figure}[htb!]
\centering
\includegraphics[width=0.7\textwidth]{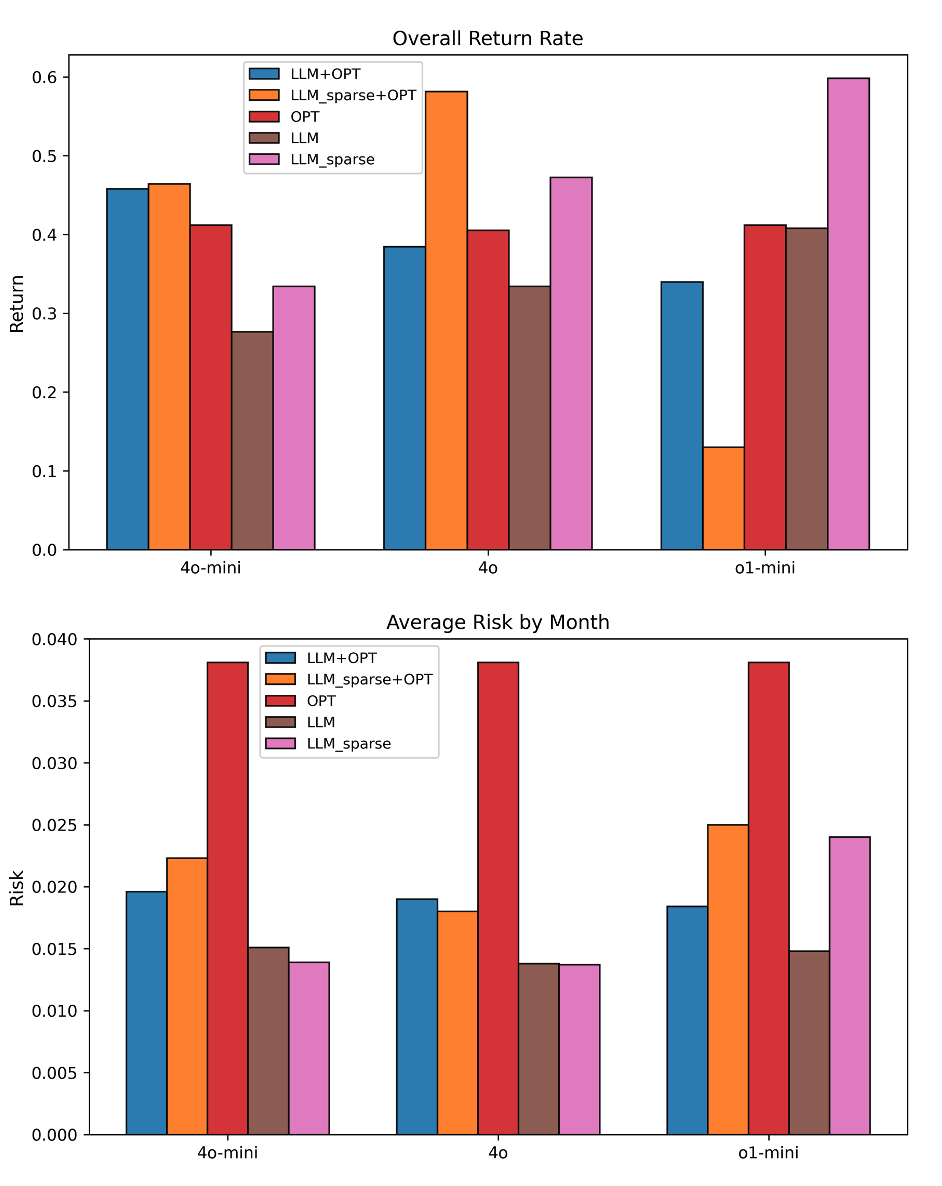}
\caption{Comparisons of overall return rate and average risk by month between ChatGPT 4o-mini, 4o and o1-mini.}
\label{fig_comparison}
\end{figure}

\newpage
\section{Prompt Used for Case Study}\label{prompt_details}

\begin{tcolorbox}[colframe=black, colback=white, coltitle=white, title=Perplexity News Prompt]
Show me key news on [Company X] (I'll provide you the stock ticker) from [Date Range]. This list isn't exhaustive—earnings, leadership changes, regulatory updates, major headlines, M\&A, industry trends, product launches, analyst opinions, investor activism, competitor moves—but only pick what's most representative for [Company X]'s stock performance. If other items seem more important, include them. Summarize in bullet points; avoid complete sentences; aim for maximum information. You don't have to include everything, just the key pieces. Keep final summary around 400 words.
\end{tcolorbox}

\begin{tcolorbox}[colframe=black, colback=white, coltitle=white, title=System Prompt]
You are collaborating with an optimization model to develop an optimal portfolio strategy. The optimization model is an expert in numerical calculations and historical trend analysis, ensuring precision in minimizing risk and achieving target returns. Your role is to bring contextual understanding, strategic insights, and adaptability to the process.

While you are encouraged to bring your unique perspective, remember:

1. The optimization model has consistently outperformed by a significant margin (10-20\%). Therefore, where substantial differences arise, you must carefully review the optimizer's proposal and adjust your recommendations unless you have a strong, data-backed reason to maintain your position.

2. In this collaboration, flexibility and compromise are critical. Adapt your strategy when the optimizer provides strong evidence, as rigid decision-making has been shown to lead to suboptimal results.

3. Your goal is to work synergistically with the optimizer to maximize overall portfolio performance. This requires understanding and integrating the optimizer's proposals to improve alignment.

4. If you disagree, explicitly identify the key aspects of the optimizer's proposal that you are willing to accept or adjust. Never reject the optimizer's strategy outright—focus on making incremental refinements.
\end{tcolorbox}

\begin{tcolorbox}[colframe=black, colback=white, coltitle=white, title=Prompt]

Recent reports indicate:

NVIDIA is a leading American semiconductor company specializing in the design and manufacture of high-end graphics processing units (GPUs). Founded in 1993 and headquartered in Santa Clara, California, NVIDIA has become a dominant force in the GPU market, holding approximately 80\% of the global market share as of 2023. The company's GPUs are widely used in gaming, professional visualization, data centers, and artificial intelligence applications. NVIDIA has positioned itself as a key player in the AI industry, with its GPUs powering many AI and machine learning platforms, including OpenAI's ChatGPT. The company's focus on innovation and rapid product development cycles has contributed to its strong market position and growth in recent years.

...

Please read the following information carefully.

---
**Stock News**

news for NVDA:
Here are the key points related to NVIDIA (NVDA) in January 2024:

\#\# Stock Performance
- NVIDIA's stock price in January 2024 saw a significant increase, closing at \$61.51 on January 31, 2024, which was up 24.94\% for the month[1][5].

\#\# Financials and Earnings
- Although the specific earnings report for Q4 FY24 was released in February 2024, the fiscal year 2024 performance was already indicative of strong growth. However, the detailed earnings report was not available in January 2024 itself[4].

\#\# Market and Industry Trends
- The surge in NVIDIA's stock was largely driven by the booming demand for graphics processing units (GPUs) due to their critical role in the generative artificial intelligence era[3].

\#\# Product and Technology Updates
- The company had introduced several innovations in the preceding and following months, such as the GeForce RTX™ 40 SUPER Series GPUs and generative AI capabilities for its installed base of RTX AI PCs, though these were not confined to January 2024[4].

\#\# General Market Sentiment
- The overall sentiment around NVIDIA in January 2024 was positive, reflecting the company's strong position in the AI and GPU markets, despite some later concerns about delays and other issues that emerged in subsequent months[2][3].

---
**Recent Stock Prices**

The stock prices today are: ...

You are a trader responsible for making portfolio allocation decisions. Use all relevant information provided (such as news and stock data) to decide how much to invest in each stock.

Think about:

1. Any news articles and how they might affect each stock.

2. Previous decisions you have made regarding portfolio weights.

Also, here is the decision-price of your plan thus far: [...]. A higher decision-price means you should adjust your plan to be higher. And a negative decision-price means you should adjust your plan to be smaller.

\#\#\# Task
1. Carefully evaluate the optimizer's proposed portfolio weights and explain your reasoning for agreement or disagreement. When in doubt, lean towards collaboration by adjusting your recommendations closer to the optimizer's.
2. Finalize your recommendation in the following format: [Ticker: Confidence Level]:
   - Very Low Confidence
   - Low Confidence
   - Somewhat Low Confidence
   - Neutral
   - Somewhat High Confidence
   - High Confidence
   - Very High Confidence

3. Conclude by summarizing how your proposal aligns with the optimizer's and why it contributes to achieving the collective goals.
Even if you are unsure, you **must** provide the best decision you can based on the available information.

Take a deep breath and work on this problem step-by-step.

\#\#\# Response Format
After your explanation, please write your final recommendation in a single line, in the format below:

NVDA: X1, AMD: X2, ...

Replace X1, X2, ... with the confidence level for each stock.

Explicitly end your response in that format. So make sure you have these stocks and confidence levels clearly written out to be parsed by a regex function.

Remember, collaboration, adaptability, and performance are key to success.

\end{tcolorbox}

\end{document}